\documentclass[twoside]{article}

\usepackage{amsmath,amsfonts,bm}

\def\eqref#1{equation~\ref{#1}}
\def\1{\bm{1}}

\def\rn{{\textnormal{n}}}

\def\rr{{\textnormal{r}}}

\def\rx{{\textnormal{x}}}
\def\ry{{\textnormal{y}}}
\def\rz{{\textnormal{z}}}

\def\rvn{{\mathbf{n}}}

\def\vzero{{\bm{0}}}

\def\vc{{\bm{c}}}

\def\vn{{\bm{n}}}

\def\vp{{\bm{p}}}
\def\vq{{\bm{q}}}

\def\vs{{\bm{s}}}

\def\vv{{\bm{v}}}

\def\vx{{\bm{x}}}
\def\vy{{\bm{y}}}
\def\vz{{\bm{z}}}

\def\mN{{\bm{N}}}

\def\mX{{\bm{X}}}

\DeclareMathAlphabet{\mathsfit}{\encodingdefault}{\sfdefault}{m}{sl}
\SetMathAlphabet{\mathsfit}{bold}{\encodingdefault}{\sfdefault}{bx}{n}

\def\sA{{\mathbb{A}}}

\usepackage[accepted]{aistats2022}
\usepackage{bm}
\usepackage{amsmath}
\usepackage{amsfonts}
\usepackage{hyperref}
\usepackage{graphicx}
\usepackage{amsthm}
\usepackage{mathtools}
\usepackage{siunitx}
\usepackage{amssymb}
\usepackage{bbm}
\usepackage{tikz}
\usepackage{caption}
\usepackage{pdfpages}
\usepackage{paralist}
\usepackage{cite}
\usepackage{dsfont}

\newtheorem{theorem}{Theorem}[section]

\newtheorem{corollary}[theorem]{Corollary}
\newtheorem{remark}[theorem]{Remark}
\newtheorem{example}[theorem]{Example}

\begin{document}

\runningtitle{When are Distances Informative for the Ground Truth in Noisy High-Dimensional Data?}

\runningauthor{Robin Vandaele, Bo Kang, Tijl De Bie, Yvan Saeys}

\twocolumn[

\aistatstitle{The Curse Revisited: When are Distances Informative for the Ground Truth in Noisy High-Dimensional Data?}

\aistatsauthor{ 
Robin Vandaele$^{1,2,3}$ 
\And Bo Kang$^3$ 
\And Tijl De Bie$^3$ 
\And  Yvan Saeys$^{1,2}$}

\aistatsaddress{ 
$^1$Department of Applied Mathematics, Computer Science and Statistics, Ghent University, Gent, Belgium\\
$^2$Data mining and Modelling for Biomedicine (DaMBi), VIB Inflammation Research Center, Gent, Belgium\\
$^3$IDLab, Department of Electronics and Information Systems, Ghent University, Gent, Belgium} ]

\begin{abstract}
Distances between data points are widely used in machine learning applications. Yet, when corrupted by noise, these distances---and thus the models based upon them---may lose their usefulness in high dimensions. Indeed, the small marginal effects of the noise may then accumulate quickly, shifting empirical closest and furthest neighbors away from the ground truth. In this paper, we exactly characterize such effects in noisy high-dimensional data using an asymptotic probabilistic expression. Previously, it has been argued that neighborhood queries become meaningless and unstable when distance concentration occurs, which means that there is a poor relative discrimination between the furthest and closest neighbors in the data. However, we conclude that this is not necessarily the case when we decompose the data in a ground truth---which we aim to recover---and noise component. More specifically, we derive that under particular conditions, empirical neighborhood relations affected by noise are still likely to be truthful even when distance concentration occurs. We also include thorough empirical verification of our results, as well as interesting experiments in which our derived `phase shift' where neighbors become random or not turns out to be identical to the phase shift where common dimensionality reduction methods perform poorly or well for recovering low-dimensional reconstructions of high-dimensional data with dense noise.
\end{abstract}

\section{Introduction}
\label{intro}

\paragraph{Motivation}

The notorious \emph{curse of dimensionality} encompasses various phenomena that occur in high-dimensional data, which complicate their analysis \cite{indyk1998approximate, beyer1999nearest, 10.1007/3-540-44503-X_27, verleysen2005curse, kuo2005lifting, Radovanovic:2009:NNH:1553374.1553485}.
In the particular case of distance functions such as Euclidean, there tends to be little contrast in the distances between different pairs of points.
This phenomenon is known as \emph{distance concentration}, and impedes learning and inference from the data through (local and global) neighborhood-based approaches.
``In other words, virtually every data point is then as good as any other, and slight perturbations to the query point would result in another data point being chosen as the nearest neighbor" \cite{beyer1999nearest}. 
Therefore, distance concentration in data is commonly regarded as indicative for the distances between points to be meaningless and the empirical neighborhood relations to be unstable.

Notwithstanding the high emphasis on `distance concentration' in the current literature, \emph{there is an entirely different yet natural possible view on `when distances are meaningful' when the data is corrupted by noise}, which we formalize in this paper.
We assume the common practical case that the observed data is composed of a \emph{ground truth} component $\mX$, and a \emph{dense noise} component $\mN$.
By `dense', we mean that each entry of $\mX$ is likely corrupted by a small nonzero error value.
We regard distances as meaningful when the (closest, furthest, $k$-nearest, \ldots) neighborhood relations derived from the observed data $\mX+\mN$ likely coincide with those that would have been obtained from $\mX$, and thus, are informative for the ground truth model underlying the data.
Since noise is unavoidable in many real world data due to practical problems in the collection and preparation processes \cite{zhu2004class}, we argue that this is a highly natural way to characterize meaningfulness of distances in such data.
For example, biological data such as single cell sequencing data \cite{zhang2021noise} is inherently noisy, due to the imprecise nature of biological experiments \cite{libralon2009pre, vandaele2021stable}.
Other high-dimensional examples include noisy images \cite{buades2005non}, climate time series \cite{ertoz2003finding}, and neuron activity data \cite{friedman2015multistage}.

While noise may contribute little to individual dimensions, its overall contribution can be especially harmful to learning from data when it is high-dimensional.
When the data dimensionality grows, and \emph{the signal, here: `the absolute difference between ground truth distances in neighborhood queries'}, cannot cope with the dense noise that is accumulated at the same time, the data will lose its discriminative power for inferring the ground truth. 
This is formally explored in this paper.

Furthermore, dimensionality reduction methods are commonly used to alleviate the effect of noise on high-dimensional data and facilitate learning.
See for example Figure \ref{kNNrandom}, where the distances are much more useful for (topological) inference from a biological cell trajectory data set after a PCA projection.
Dimensionality reductions are also used to obtain more meaningful distances in applications such as spectral clustering \cite{liu2004spectral}, and even prior to other embedding methods such as t-SNE \cite{van2008visualizing}.
This suggests the need of a formal exploration of how dimensionality reductions themselves are susceptible to noise in high dimensions, which we provide in this paper.
While we only study this empirically for synthetic examples within the limited scope of this paper, the fact that \emph{our derived `phase shift' where neighbors become non-random is identical to the phase shift where common dimensionality reductions methods start performing well}, encourages further theoretical and methodological research into this subject.

\begin{figure}[t]
	\centering
	\includegraphics[width=.7\linewidth]{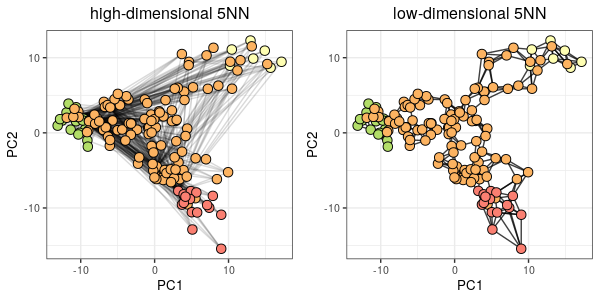}
	\caption{5NN graphs  (edges in black) of a real bifurcating biological cell trajectory (gene expression) data set $\mX+\mN \subseteq \mathbb{R}^{1770}$ consisting of four cell groups \cite{robrecht_cannoodt_2018_1443566, Saelens276907, vandaele2020mining}. 
	Each dimension quantifies the expression of a particular gene, and the coloring of each point (cell) corresponds to its cell group.
	The 5NN graphs are visualized through a 2-dimensional PCA embedding of $\mX+\mN$.
	(Left) The edges of the 5NN graph are obtained directly from the distances between the high-dimensional points.
	(Right) The edges of the 5NN graph are obtained from the 2-dimensional PCA embedding.
	The possible placement of points is much more constrained in two dimensions.
	This reduces unwanted behavior caused by noisy high-dimensional distances that impede trajectory inference, such as interconnections between different branches.
	The resulting lower-dimensional representation will be much more effective for learning the bifurcating model, as common in trajectory inference \cite{Saelens276907}.}
	\label{kNNrandom}
\end{figure}

Note that we will not introduce any novel algorithms in this paper. 
Yet, we do present and validate novel theoretical results about learning from high-dimensional point cloud data. 
These results complement previous work on distance concentration, and add to the understanding of counter-intuitive phenomena of high-dimensional data. 
We argue that \emph{such improved understanding is imperative for the design of better computational methods for the analysis of such data}.

Finally, we emphasize that \emph{the role of distance concentration in this paper is to be interpreted as rather subtle}.
In fact, this paper could be (and starting from Section \ref{characterize} is for a large part) written independent from distance concentration.
However, given its prevalence in related work, we found it important to include distance concentration in our motivation and discussion, and point out its differences to the view on `meaningfulness in distances' analyzed in this paper. 

\paragraph{Related Work}

It is well known that many distance measures lose their usefulness for discriminating between neighbors in high-dimensional data. 
This phenomenon, generally known as distance concentration, has been studied extensively on both a theoretical and experimental level \cite{beyer1999nearest,10.1007/3-540-44503-X_27, durrant2009nearest, kaban2012non, giannella2021instability}.
Its occurrence in data is widely perceived as an indicator that distances between observations are meaningless and neighborhood queries are unstable.

However, we argue that for the abundance of real word data with noise, \emph{`distances are meaningful when they are informative for the ground truth'} is a natural point of view.
Nevertheless, to the best of our knowledge, a formal probabilistic analysis of such characterization is lacking, as even extensive studies on the behavior of distances and neighbors in high-dimensional data \cite{angiulli2017behavior} do not include an analysis that explicitly separates a ground truth from a noise component.
By maintaining this strategy in this paper however, we make important conclusions that add to the understanding of high-dimensional data.
These include that our proposed view on meaningful distances cannot be characterized through distance concentration, and that neighborhood relations may still remain truthful even when distances are dominated by noise.

\paragraph{Contributions}

\begin{compactitem}
	\item We provide a probabilistic quantification of the effect of high-dimensional noise on neighboring relations, deriving conditions under which these relations either become highly random or non-random, independent of the magnitude of noise (Theorem \ref{CLTspecial} \& Corollary \ref{randomornot}).
	\item We provide thorough empirical verification of our theoretical results, and show that our novel yet natural view on meaningful distances is different from distance concentration (Section \ref{validate}).
	\item We use hyperharmonic series (Example \ref{hyperharmonic}) to develop experiments that directly link the performance of dimensionality reductions to the randomness of neighborhood relations (Section \ref{overfit}).
	\item We conclude on how our work provides better understanding of learning from noisy high-dimensional point cloud data (Section \ref{discandcon}).
\end{compactitem}

\section{Quantifying the Effect of Noise on High-Dimensional Neighbors}
\label{characterize}

In the first part of this section (Section \ref{whensection}), we provide a probabilistic quantification of the effect of noise on the absolute discrimination between high-dimensional neighbors (Theorem \ref{CLTspecial}), and use this to deduce conditions under which empirical neighbors become either highly random or likely truthful (Corollary \ref{randomornot}).
In particular, it will follow that these conditions are independent of the magnitude of noise in the dimensions.
Although we commence the analysis assuming we have three fixed points $\vx,\vy,\vz\in \mX$ in a given ground truth data set $\mX$, in Section \ref{liftToX} we also discuss how our obtained results can be used to derive more general results for $\mX$, such as Theorem \ref{liftToRandom}. 

\subsection{When Neighbors Become (non-)Random: a Case Study for Three Points}
\label{whensection}

Our setting in this section will be as follows.

\begin{itemize}
	\item We are given three sequences (which are to be interpreted as vectors) $\vx=x_1,x_2$,\ldots, $\vy=y_1,y_2,\ldots$, and $z=z_1, z_2,\ldots$.
	These correspond to the ground truth---and thus in practice---non-observed points.
	For $d\in\mathbb{N}^*$, $x_d$ equals the information captured by the $d$-th dimension of $\vx$ (analogous for $\vy,\vz$).
	If the model is explained by a finite number of dimensions, we can still regard $\vx,\vy,\vz$ as infinite sequences by letting $x_d=y_d=z_d$ (e.g. $=0$) for additional dimensions $d$.
	\item Rather than observing the vectors $\vx,$ $\vy$, and $\vz$, we observe $\vx+\vn_{\vx}$, $\vy+\vn_{\vy}$, and $\vz+\vn_{\vz}$.
	Here, $\vn_{\vx}$ is a realization of a sequence of random noise variables $\rvn_{\vx}=\rn_{x_1},\rn_{x_2},\ldots$ (analogous for $\vy,\vz$).
	We will assume the random variables $\bigcup_{d\in\mathbb{N}^*}\{\rn_{x_d},\rn_{y_d},\rn_{z_d}\}$ to be i.i.d, have finite fourth moment $\mu'_4$ (measuring the heaviness of the tail of the noise distribution), and be symmetric.
	While the former two assumptions will be required by the analysis, the latter simply makes it more convenient.
	Nevertheless, many common random noise distributions such as uniform and normal, are symmetric.
\end{itemize}

The following result should be interpreted as follows.
We are given a query point $\vx$ from a ground truth data set $\mX$, and two candidate neighbors $\vy$ and $\vz$ of $\vx$.
We want a formula expressing how likely neighborhood relations between $\vx$, $\vy$, and $\vz$, such as `\emph{$\vx$ is closer to $\vy$ than to $\vz$}', are preserved after introducing additive noise $\mN$, i.e., we observe $\mX+\mN$ rather than $\mX$.
This formula should be asymptotically valid, i.e., for a sufficiently high dimensionality $d$ of $\mX$.
Intuitively, the resulting probabilities will be in terms of the true distances $\|\vx-\vy\|$ and $\|\vx-\vz\|$, and noise characteristics, here $\sigma^2$ and $\mu'_4$.
Indeed, when there is not much difference between $\|\vx-\vy\|$ and $\|\vx-\vz\|$ (\emph{the signal}), or when $\sigma^2$ and $\mu'_4$ are large, we expect it to be more difficult to preserve neighborhood relations.
The following formula will then be used to derive subsequent insightful results in the rest of this paper.

\begin{theorem}
	\label{CLTspecial}
	Let $\vx=x_1,x_2,\ldots$, $\vy=y_1,y_2,\ldots$ and $\vz=z_1,z_2,\ldots$ be three sequences of real numbers.
	Let $\rvn_{\vx}=\rn_{x_1},\rn_{x_2},\ldots$, $\rvn_{\vy}=n_{y_1},n_{y_2},\ldots$, and $\rvn_{\vz}=n_{z_1},n_{z_2},\ldots$ be three sequences of jointly i.i.d.\ symmetric continuous random variables with variance $\sigma^2$ and finite $4$th moment $\mu'_4$.
	For a sequence $\vs$, denote $\vs^{(d)}$ for the vector composed from its first $d$ elements in order.
	Finally, let
	$$
	\Delta_\infty(d)\coloneqq \max\left\{\left\|\vx^{(d)}-\vy^{(d)}\right\|_\infty,\left\|\vx^{(d)}-\vz^{(d)}\right\|_\infty\right\}.
	$$
	If 
	\begin{align}
	\label{condition}
	\lim_{d\rightarrow\infty}\frac{\Delta_{\infty}(d)}{\sqrt[^{16}]{d}}=0,
	\end{align}
	then
	\scriptsize
		\begin{align*}
		&\left|P\left(\left\|\vx^{(d)}+\rvn^{(d)}_{\vx}-\vy^{(d)}-\rvn^{(d)}_{\vy}\right\|\leq \left\|\vx^{(d)}+\rvn^{(d)}_{\vx}-\vz^{(d)}-\rvn^{(d)}_{\vz}\right\|\right)\right.\\
		&\hspace{12.5em}\left.-\Phi\left(\zeta^{(d)}\left(\mu'_4,\sigma,\vx,\vy,\vz\right)\right)\right|\overset{d\rightarrow\infty}{\longrightarrow} 0,
		\end{align*}
	\normalsize
	where $\zeta^{(d)}:\left(\mathbb{R}^+\right)^2\times\left(\mathbb{R}^d\right)^3\rightarrow\mathbb{R}:$
	\begin{equation}
    \label{zeta}
	    {\small
	    \begin{aligned}
        \begin{pmatrix}
        \mu'_4\\
        \sigma\\
        \vx\\
        \vy\\
        \vz
        \end{pmatrix}
	    \mapsto\tfrac{\left\|\vx-\vz\right\|^2-\left\|\vx-\vy\right\|^2}{\sqrt{2d\left(\mu'_4+3\sigma^4\right) + 8\sigma^2\left(\left\|\vx-\vy\right\|^2+\left\|\vx-\vz\right\|^2-\left\langle \vx - \vy, \vx - \vz\right\rangle\right)}},
	\end{aligned}}
	\end{equation}
	and $\Phi$ is the cumulative distribution function of the standard normal distribution.
\end{theorem}

Letting
\begin{equation}
\label{Zformula}
    {\scriptsize
    \begin{aligned}
    \rz(d)&\coloneqq\left\|\rvn^{(d)}_{\vx}-\rvn^{(d)}_{\vy}+\vx^{(d)}-\vy^{(d)}\right\|^2\\
    &\hspace{2em}-\left\|\rvn^{(d)}_{\vx}-\rvn^{(d)}_{\vz}+\vx^{(d)}-\vz^{(d)}\right\|^2\\
    &=\sum_{i=1}^d\underbrace{\left[\left(\rn_{x_i}-\rn_{y_i}+x_i-y_i\right)^2-\left(\rn_{x_i}-\rn_{z_i}+x_i-z_i\right)^2\right]}_{\eqqcolon \rz_i},
    \end{aligned}}
\end{equation}
the proof of Theorem \ref{CLTspecial} is based on an application of the central limit theorem (CLT) to quantify the limiting behavior of $\rz(d)$.
However, the random variables $\rz_i$, $i=1,\ldots, d$, are not necessarily identically distributed, i.e., with the same mean and variance.
For this reason, unlike the analysis by \cite{beyer1999nearest, 10.1007/3-540-44503-X_27}, we require special conditions to ensure that the CLT remains applicable in our result.
To this end, (\ref{condition}) provides a sufficient condition for Linderberg's condition to be satisfied \cite{lindeberg1922neue}.
A full proof of Theorem \ref{CLTspecial} is provided in Appendix \ref{thandproofs}.

\begin{remark}
	\label{otherconditions}
	Following the proof in Appendix \ref{thandproofs}, condition (\ref{condition}) can be further weakened to:
% 	\begin{equation}
% 	\begin{aligned}
% 	\label{other1}
% 	&\left(\tfrac{\Delta^4_{\infty}(d)}{d+\left\|\vx^{(d)}-\vy^{(d)}\right\|\left\|\vx^{(d)}-\vz^{(d)}\right\|}\overset{d\rightarrow\infty}{\longrightarrow} 0\right)\\
% 	&\hspace{.25em}\mbox{and }\left(\tfrac{\min\left\{d\Delta^4_{\infty}(d),\left\|\vx^{(d)}-\vy^{(d)}\right\|^2\left\|\vx^{(d)}-\vz^{(d)}\right\|^2\right\}}{\left(d+\left\|\vx^{(d)}-\vy^{(d)}\right\|\left\|\vx^{(d)}-\vz^{(d)}\right\|\right)^\frac{5}{4}}\overset{d\rightarrow\infty}{\longrightarrow} 0\right),
% 	\end{aligned}
% 	\end{equation}
% 	or
    \begin{equation}
    \label{other2}
        {\scriptsize
	    \begin{aligned}
	    &\left(
	    \frac{\Delta^4_{\infty}(d)}{d+\left\|\vx^{(d)}-\vy^{(d)}\right\|\left\|\vx^{(d)}-\vz^{(d)}\right\|}\overset{d\rightarrow\infty}{\longrightarrow} 0\right){\normalsize \mbox{ and }} \left(\forall\epsilon>0\right)\\
	    &
	    \begin{pmatrix}
	    \frac{\splitfrac{F_{\rn}\left(-\left(d+\left\|\vx^{(d)}-\vy^{(d)}\right\|\left\|\vx^{(d)}-\vz^{(d)}\right\|\right)^{\frac{1}{4}}\epsilon\right)}{\times\min\left\{ d\Delta^4_{\infty}(d),\left\|\vx^{(d)}-\vy^{(d)}\right\|^2\left\|\vx^{(d)}-\vz^{(d)}\right\|^2\right\}}}{\displaystyle d+\left\|\vx^{(d)}-\vy^{(d)}\right\|\left\|\vx^{(d)}-\vz^{(d)}\right\|}\overset{d\rightarrow\infty}{\longrightarrow}0
	    \end{pmatrix},
	    \end{aligned}}
	\end{equation}
	where $F_{\rn}$ is the cumulative distribution function of the (marginal) random noise variable $\rn$.
% 	By making use of Markov's inequality, it can be straightforwardly shown that (\ref{condition}) $\implies$ (\ref{other1}) $\implies$ (\ref{other2}).
	By making use of Markov's inequality, it can be straightforwardly shown that (\ref{condition}) $\implies$ (\ref{other2}).
	While this condition is less insightful than (\ref{condition}), it can be used to easily show that it suffices that $\Delta_{\infty}(d)=o\left(d^{\frac{1}{4}}\right)$ in the generic case that $\rn$ is bounded.
	Nevertheless, in the practical case that $\Delta_\infty(d)$ is bounded, i.e., when newly added dimensions are (eventually) at most as discriminating as the former, condition (\ref{condition}) is trivially satisfied.
	In Corollary \ref{randomornot} we will assume such bound, as it allows for a convenient way to `symmetrize' the asymptotic growth conditions formalized in this result.
\end{remark}

Under the same setting as for Theorem \ref{CLTspecial}, i.e., given a query point $\vx$ and two candidate (closest, furthest, \ldots) neighbors $\vy$ and $\vz$ of $\vx$, we can now study growth conditions on the signal---this being how well we can discriminate between $\vy$ and $\vz$ as the ground truth neighbors of $\vx$---under which the signal `beats' the noise in high dimensions and vice versa.
If the noise beats the signal, then the empirical neighborhood relations, i.e., those derived after additive noise is introduced, thus from the observed data, will be (nearly) completely random.
This is expressed by Corollary \ref{randomornot}.1 below.
In the opposite case, the signal beats the noise, and the empirical neighborhood relations will (likely) agree with those that would have been derived without noise, i.e., from the ground truth points $\vx$, $\vy$, and  $\vz$.
This is expressed by Corollary \ref{randomornot}.2.
The proofs of these results are provided in Appendix \ref{thandproofs}.

\begin{corollary}
	\label{randomornot}
	Let $\vx=x_1,x_2,\ldots$, $\vy=y_1,y_2,\ldots$ and $\vz=z_1,z_2,\ldots$ be three sequences of real numbers.
	Let $\rvn_{\vx}=\rn_{x_1},\rn_{x_2},\ldots$, $\rvn_{\vy}=\rn_{y_1},\rn_{y_2},\ldots$, and $\rvn_{\vz}=\rn_{z_1},\rn_{z_2},\ldots$ be three sequences of jointly i.i.d.\ symmetric continuous random variables with finite $4$th moment $\mu'_4$.
	Suppose further that $\sup_{d\in\mathbb{N}^*}\Delta_\infty(d)\leq C$ for some constant $C$.
	Then the following two statements are true.
	\begin{enumerate}
		\item If 
		$\left\|\vx^{(d)}-\vz^{(d)}\right\|^2-\left\|\vx^{(d)}-\vy^{(d)}\right\|^2=o\left(d^{\frac{1}{2}}\right)$, 
		\begin{align*}
			&\lim_{d\rightarrow\infty}P\left(\left\|\vx^{(d)}+\rvn^{(d)}_{\vx}-\vy^{(d)}-\rvn^{(d)}_{\vy}\right\|\leq\right.\\
			&\hspace{5em}\left.\left\|\vx^{(d)}+\rvn^{(d)}_{\vx}-\vz^{(d)}-\rvn^{(d)}_{\vz}\right\|\right)=\frac{1}{2}.
		\end{align*}
		\item If
		$d^{\frac{1}{2}}=o\left(\left\|\vx^{(d)}-\vz^{(d)}\right\|^2-\left\|\vx^{(d)}-\vy^{(d)}\right\|^2\right)$,
		\begin{align*}
		&\lim_{d\rightarrow\infty}P\left(\left\|\vx^{(d)}+\rvn^{(d)}_{\vx}-\vy^{(d)}-\rvn^{(d)}_{\vy}\right\|\right.\\
		&\hspace{5em}\left.\leq \left\|\vx^{(d)}+\rvn^{(d)}_{\vx}-\vz^{(d)}-\rvn^{(d)}_{\vz}\right\|\right)=1.
		\end{align*}
	\end{enumerate}
\end{corollary}

\begin{remark}
	\label{interesting}
	The following conclusions---which will all be validated in Section \ref{experiments}---can now be made.
	
	\begin{enumerate}
		\item $\Delta_2(d)\coloneqq\max\left\{\left\|\vx^{(d)}-\vy^{(d)}\right\|,\left\|\vx^{(d)}-\vz^{(d)}\right\|\right\}$ has to grow at least as $d^\frac{1}{4}$, otherwise neighbors will become highly random.
		Indeed, if $\Delta_2(d)=o\left(d^\frac{1}{4}\right)$, then condition (\ref{other2}) in Remark \ref{randomornot} is satisfied, and Corollary \ref{randomornot}.1 becomes applicable (note that the assumed bound $\sup_{d\in\mathbb{N}^*}\Delta_\infty(d)\leq C$ is only required for Corollary \ref{randomornot}.2, see also Appendix \ref{thandproofs}).
		\item The noise characteristics $\sigma^2$ and $\mu'_4$ naturally have a direct effect on the usefulness of empirical neighbors for any fixed dimension $d$, which is intuitively clear, and can be seen from (\ref{zeta}).
		However, under the conditions of Corollary \ref{randomornot}, their magnitude becomes negligible in large dimensions.
		\item Even if the distances between the noise vectors are dominant in the distances between the observed data points, neighbors may become non-random, i.e., representative for the ground truth neighbors.
		For example, if the random noise variable $\rn$ is uniformly distributed, then the expected distances between two noise vectors grows as $\sqrt{d}$, whereas it is sufficient for the absolute differences between ground truth distances to grow as $\sqrt[^4]{d}+\epsilon$ for some $\epsilon > 0$ to ensure that the noise becomes unlikely to effect neighbors in sufficiently high dimensions.
	\end{enumerate}
\end{remark}

The following example will prove to be very useful in the experiments (Section \ref{experiments}).

\begin{example}
	\label{hyperharmonic}
	Let $\vx=\vy$ be the sequences of all zeros.
	Given $\alpha \in \mathbb{R}_{>2}\cup\{\infty\}$, we define the sequence of reals $\vz(\alpha)=z_1(\alpha),z_2(\alpha),\ldots$, by letting for $d\in\mathbb{N}^*$, 
	$$
	z_d(\alpha)\coloneqq
	\begin{cases}
	\frac{1}{\sqrt[^\alpha]{d}}&\mbox{if }\alpha\in\mathbb{R}_{>2};\\
	1&\mbox{if }\alpha=\infty.
	\end{cases}
	$$
	Observe that we have $\Delta_{\infty}(d)=\left\|\vz^{(d)}(\alpha)\right\|_{\infty}=1$ for all $d\in\mathbb{N}^*$.
	We furthermore find that for $\alpha\in\mathbb{R}_{>2}$, $\left\|\vz^{(d)}(\alpha)\right\|^2$ defines a \emph{hyperharmonic series}, which for large $d\in\mathbb{N}^*$ can be approximated as
	\begin{align}
	    \label{alpha}
	    \begin{split}
	    &\left\|\vz^{(d)}(\alpha)\right\|^2=\sum_{k=1}^d \frac{1}{\sqrt[^\alpha]{k^2}}\\
	    &\hspace{2.5em}\overset{d\rightarrow\infty}{\sim}\int_{1}^d x^{-\frac{2}{a}}\mathrm{d}x=\left(\frac{\alpha}{\alpha-2}\right)\left(d^{1-\frac{2}{\alpha}}-1\right),
	    \end{split}
    \end{align}
	i.e., $\left\|\vz^{(d)}(\alpha)\right\|^2$ grows as $d^{1-\frac{2}{\alpha}}$.
	With by convention $\frac{1}{\infty}=0$, this holds for $\alpha=\infty$ as well.
	Naturally, in any dimension $\vy$ is always closer to $\vx$ than $\vz$ is.
	Due to Corollary \ref{randomornot}, the probability that this remains true under noise in high dimensions satisfies
    \begin{align*}
    	&P\left(\left\|\vx^{(d)}+\rvn^{(d)}_{\vx}-\vy^{(d)}-\rvn^{(d)}_{\vy}\right\|\right.\\
    	&\hspace{3em}\left.\leq \left\|\vx^{(d)}+\rvn^{(d)}_{\vx}-\vz^{(d)}-\rvn^{(d)}_{\vz}\right\|\right)\\
    	&\hspace{7em}\overset{d\rightarrow\infty}{\longrightarrow}
    	\begin{cases}
    	\frac{1}{2}&\alpha<4;\\
    	1&\alpha>4;\\
    	\Phi\left(\sqrt{\frac{2}{\mu'_4+3\sigma^4}}\right)&\alpha = 4.
    	\end{cases}
    \end{align*}
	The last limit can be found from (\ref{alpha}) by adapting the proof of Corollary \ref{randomornot}.2 in Appendix \ref{thandproofs}.
	Thus, $\alpha=4$ corresponds to a `\emph{phase shift}', where neighbors transition between becoming random or non-random.
\end{example}

In the following section, we discuss how our results can be used to derive more general results for larger data $\mX$, consisting of more than three points $\vx,\vy$, and $\vz$.

\subsection{Randomness in Neighbors for Data Sets of Arbitrary Sizes}
\label{liftToX}

In the previous section we restricted to the particular scenario where we have three given points $\vx$, $\vy$, and $\vz$.
Naturally, we can also study the effectiveness of deriving neighborhood relations in a data set $\mX$ of arbitrary size under the effect of noise in high dimensions.
The reason for this is that in practice, we deal with a finite number of data points.
Therefore, more general mathematical results on preserving neighbors under noise may often be derived directly from the results presented in Section \ref{whensection}.

One such example is as follows, providing sufficient growth conditions on the ground truth diameter for the noise to cause neighboring relations to become random, or thus necessary conditions for the noise not to cause this (which is what we want to achieve in practice).
The idea here is that when for every point $\vx\in \mX$, if in the empirical noisy data it is completely random whether $\vx$'s true furthest neighbor becomes closer to $\vx$ than $\vx$'s true closest neighbor, i.e., this occurs with probability $\frac{1}{2}$, then one can essentially not work effectively with the high-dimensional neighbors.

\begin{theorem}
	\label{liftToRandom}
	Let $\mX=\vv_1,\vv_2,\ldots$ be a sequence of column vectors, for which $\vv_d\in\mathbb{R}^{n}$, $n\in\mathbb{N}^*$, and denote by $\mX^{(d)}$ the matrix in $\mathbb{R}^{n\times d}$ composed of the first $d$ vectors $\vv_d$ in order.
	For $d\in\mathbb{N}^*$, and $i=1,\ldots,n$, we identify the $i$-th row of the matrix $\mX^{(d)}$ with the point $\vx^{(d)}_i\in \mX^{(d)}$.
	Let furthermore $\{\rvn_{i}=\rn_{i_1},\rn_{i_2},\ldots\}_{i\in\{1,\ldots,n\}}$ be a collection of jointly i.i.d.\ symmetric continuous random variables with finite $4$th moment.
	For $d\in\mathbb{N}^*$ and each point $\vx^{(d)}_i\in \mX^{(d)}$, let $\vx^{(d)}_{i,\min}$ denote the closest neighbor of $\vx^{(d)}_i$ in $\mX^{(d)}$, $\vx^{(d)}_{i,\max}$ the furthest neighbor of $\vx^{(d)}_i$ in $\mX^{(d)}$, and ${\bf{\Delta_2}}(d)\coloneqq\max_{\vx,\vy\in \mX^{(d)}}\|\vx-\vy\|$ the diameter of $\mX^{(d)}$.
	If ${\bf{\Delta_2}}(d)=o\left(d^{\frac{1}{4}}\right)$, then
	\begin{align*}
	{\small
	    \begin{aligned}
	    &\sup_{i=1,\ldots,n}P\left(\left\|\vx^{(d)}_i+\rvn^{(d)}_i-\vx^{(d)}_{i,\max}-\rvn^{(d)}_{\vx^{(d)}_{i,\max}}\right\|\right.\\
	    &\hspace{5.5em}\left.\leq\left\|\vx^{(d)}_i+\rvn^{(d)}_i-\vx^{(d)}_{i,\min}-\rvn^{(d)}_{\vx^{(d)}_{i,\min}}\right\|\right)\overset{d\rightarrow\infty}{\longrightarrow}\frac{1}{2}.
        \end{aligned}}
	\end{align*}
\end{theorem}

\begin{proof}
	Corollary \ref{randomornot}.1 is valid for every triple of rows in $\mX$ (see also Remark \ref{interesting}.1).
	Hence, the result follows from the fact that there are only finitely many such triples (note that the row indices of $\vx^{(d)}_{i,\min}$ and $\vx^{(d)}_{i,\max}$ in $\mX^{(d)}$ are allowed to vary with $d$).
\end{proof}

Theorem \ref{liftToRandom} essentially states the conditions we must avoid for any practical application that relies on the distances between data observations.
Conversely, the conditions we should aim for are generally more dependant on the application of interest. 
Indeed, for many practical purposes it may not be important to preserve \emph{all} neighborhood relations in the data.
For example, as will be validated in Section \ref{overfit}, for clustering algorithms a sufficient condition for cluster assignments to be likely truthful would be that Corollary \ref{randomornot}.2 is satisfied for every triple $(\vx,\vy,\vz)$, where $\vx$ and $\vy$ belong to the same (ground truth) cluster and $\vz$ to a different cluster.
Under the effect of noise, any two points from the same ground truth cluster then likely remain closer to each other than to any point from another cluster.
However, for particular methods such as single-linkage clustering, this would be too stringent, and more in-depth analysis will be required.
Thus, we will not claim one such generally applicable result in the current paper.
Nevertheless, the consensus remains that to overcome the impact of the extra noise, adding dimensions should be accompanied with adding sufficient information to discriminate between important neighbors according to the ground truth.

\section{Experimental Results}
\label{experiments}

In this section we conduct experiments that aim to improve one's understanding and intuition about working with distances in noisy high-dimensional data.
First, Section \ref{validate} will be devoted to empirical validation of our theoretical results.
Section \ref{overfit} will be devoted to linking the performance of common dimensionality reduction methods as well as spectral clustering to randomness of neighborhood relations.
While our empirical observations in Section \ref{overfit} cannot be immediately derived from our theoretical results in their current stage, they point out direct and interesting connections between the performance of machine learning algorithms and the signal-to-noise ratio as formalized in this paper, which encourage further research into this subject.
Code for this project is available on \url{https://github.com/robinvndaele/NoisyDistances}.

\subsection{Validation of the Theoretical Results}
\label{validate}

\paragraph{Validation of Theorem \ref{CLTspecial}}

We constructed three sets of three sequences $\vx$, $\vy$, and $\vz$ containing up to \num{10000} dimensions.
For each set, we let $\vx=\vy=\vzero$. 
$\vz$ is used to control ground truth distance growth rates, here measured through the $l_2$ and $l_\infty$ norm, as follows.

\begin{compactitem}
	\item[1.] \textbf{$\bm{l_2}$ bounded, $\bm{l_\infty}$ bounded:} $\vz=(1,0,\ldots,0)$.
	\item[2.] \textbf{$\bm{l_2}$ unbounded, $\bm{l_\infty}$ bounded:} $\vz=(1,1,\ldots,1)$.
	\item[3.] \textbf{$\bm{l_2}$ unbounded, $\bm{l_\infty}$ unbounded:} $z_d=d^{\frac{1}{4}-0.01}$.
\end{compactitem}

\begin{figure}[t!]
	\centering
	\includegraphics[width=.5\linewidth]{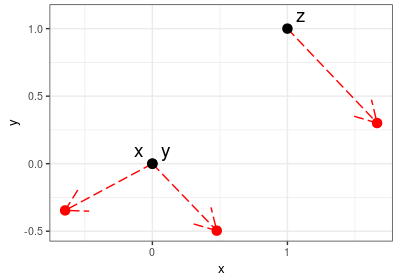}
	\caption{The experiment setup visualized in 2D.
		We have three points $\vx=\vy$ and $\vz$, where $\vz$ controls the true discrimination growth rate.
		Theorem \ref{CLTspecial} quantifies how likely $\vx$ will remain closer to $\vy$ than to $\vz$ when corrupted by noise in high dimensions, here illustrated by the displacements in red, for various growth rates.}
	\label{setupnoise}
	\includegraphics[width=\linewidth]{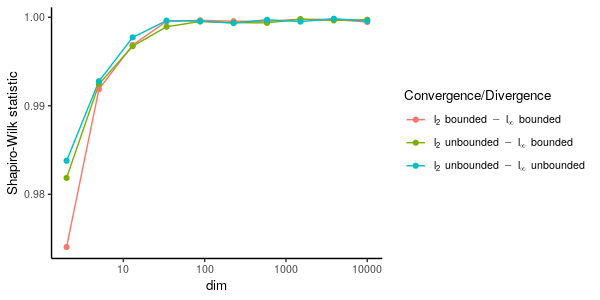}
	\caption{Shapiro–Wilk test statistics to assess normality of $\ry(d)$ for the various growth rates determined by $\vz$, according to the data dimension $d$.
	The convergence of the curves to 1 agrees that Theorem \ref{CLTspecial} is applicable to all of the growth rates, and thus allows us to quantify the randomness of neighborhood relations between $\vx,\vy$, and $\vz$ caused by noise.}
	\label{CheckCLT}
    \includegraphics[width=.9125\linewidth]{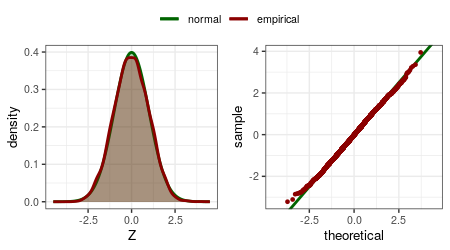}
	\caption{Density plot and Q-Q plot of $\ry(\num{10000})$ (empirical) compared to the standard normal distribution.}
	\label{NormPlotHigh}
\end{figure}

Then for each sequence and in each dimension we added uniform noise $\rn\sim \mathrm{U}[-0.75, 0.75]$, 
for which $\sigma^2=\frac{0.75^2}{3}$ and $\mu'_4=\frac{0.75^4}{5}$.
The setup for this experiment is illustrated by Figure \ref{setupnoise}. 
Since $\rn$ is bounded in each dimension, Theorem \ref{CLTspecial} should be applicable to all three cases (see also Remark \ref{otherconditions}).
More precisely, for sufficiently large $d$ we should find that $\ry(d)\coloneqq\frac{\rz(d)-\mu\left(\rz(d)\right)}{\sigma\left(\rz(d)\right)}\overset{\mathrm{approx.}}{\sim}N(0,1)$, where $\rz(d)$ is as defined in (\ref{Zformula}), and $\mu\left(\rz(d)\right)$ is the negative nominator and $\sigma\left(\rz(d)\right)$ the denominator of $\zeta^{(d)}\left(\mu'_4,\sigma,\vx,\vy,\vz\right)$ in (\ref{zeta}).
We verified this through 5000 samples of $\ry(d)$ for each of the three ground truth growth rates and various dimension $d$ chosen from a log-scale (replicating the noise outcomes).
We used the Shapiro–Wilk test to assess normality.
The results are shown in Figure \ref{CheckCLT}, confirming that Theorem \ref{CLTspecial} is indeed applicable to all of the considered growth rates.
Since the Shapiro–Wilk test is developed to asses normality, but not  \emph{standard} normality, Figure \ref{NormPlotHigh} shows normality plots of $\ry(\num{10000})$ for the second set of sequences, confirming the correctness of our calculations.

\paragraph{Validation of Corollary \ref{randomornot}}

We will use Example \ref{hyperharmonic} to analyze for which growth rates of the true discrimination between neighbors (the signal), empirical neighbors become random or not.
For this, we considered various sets of three sequences $\vx,\vy$, and $\vz$, where $\vx=\vy=\vzero$, and $\vz=\vz(\alpha)$ controls the growth rate as determined by $\alpha$ in Example \ref{hyperharmonic}.
We sampled noise using a uniform distribution $\rn\sim \mathrm{U}[-1.25, 1.25]$.
A higher magnitude of noise is chosen here to better illustrate that $0<\Phi_4\coloneqq\Phi\left(\sqrt{\tfrac{2}{\mu'_4+3\sigma^4}}\right)<1$ (for $\alpha\neq 4$ the magnitude will not matter in the limit).
The setup of this experiment is again visualized by Figure \ref{setupnoise}.
We used 5000 noise replicates to approximate a variety of expected values and probabilities for different growth rates determined by $\alpha\in\{2,3,4,5,6,\infty\}$.
These are illustrated on Figure \ref{AlphaExp}, and defined as follows.

\begin{figure}[b!]
    \centering
	\includegraphics[width=\linewidth]{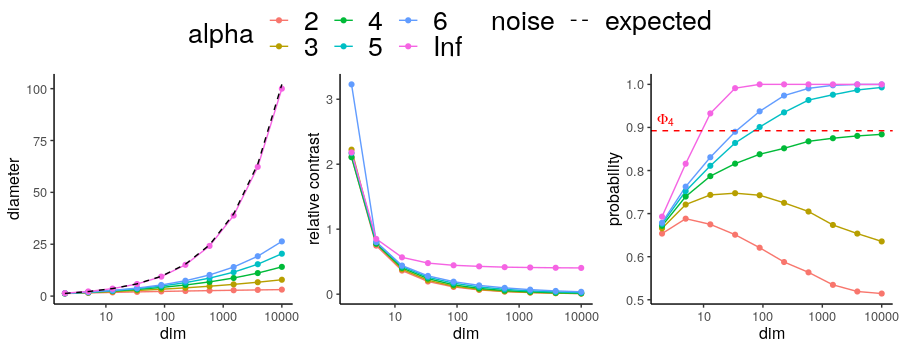}
	\caption{(Left) The expected growth rate of the distance between two noise vectors compared to the ground truth model diameter growth rates determined by $\alpha$.
		(Middle) The expected relative contrast converges to 0 for all considered values $\alpha<\infty$, meaning that the two closest points are expected to be relatively as distant to each other as the two furthest points.
		(Right) The limiting behavior of $P\left(\left\|\vx+\rvn_{\vx}-\vy-\rvn_{\vy}\right\|\leq \left\|\vx+\rvn_{\vx}-\vz-\rvn_{\vz}\right\|\right)$.
		Under noise in high dimensions, neighbors will be inferred effectively for $\alpha>4$, i.e., empirical neighbors will likely be true, whereas for $\alpha<4$, they become meaningless.}
	\label{AlphaExp}
\end{figure}

\begin{enumerate}
	\item The expected distance between two noise vectors, compared to the ground truth diameter growth rates, i.e., of $\|\vz\|$ (Figure \ref{AlphaExp}, Left).
	\item The expected \emph{relative contrast} (Figure \ref{AlphaExp}, Middle) \cite{10.1007/3-540-44503-X_27}:
	$$\frac{\max_{\vp,\vq\in\{\vx,\vy,\vz\}}\|\vp+\rvn_{\vp}-\vq-\rvn_{\vq}\|}{\min_{\vp,\vq\in\{\vx,\vy,\vz\}}\|\vp+\rvn_{\vp}-\vq-\rvn_{\vq}\|}-1.$$
	A relative contrast near 0 indicates distance concentration (discussed in Section \ref{intro}).
	\item The probability (Figure \ref{AlphaExp}, Right)
	$$P\left(\left\|\vx+\rvn_{\vx}-\vy-\rvn_{\vy}\right\|\leq \left\|\vx+\rvn_{\vx}-\vz-\rvn_{\vz}\right\|\right).$$
\end{enumerate}

First, we observe that the distances between the noise vectors is expected to become indefinitely larger than the distances between the ground truth points for $\alpha<\infty$ (Figure \ref{AlphaExp}, Left).
Hence, we would intuitively expect the distances between the noise vectors to play a dominant role in the observed empirical distances for the corresponding growth rates.
Second, we observe that also for all considered growth rates determined by $\alpha<\infty$, the expected relative contrast converges to $0$ (Figure \ref{AlphaExp}, Middle).
This means that in sufficiently high dimensions, the two closest points are expected to be relatively as distant to each other as the two furthest points.
If one would interpret this as neighborhood queries to become meaningless and unstable---as argued by \cite{beyer1999nearest, 10.1007/3-540-44503-X_27}---according to our view discussed in Section \ref{intro}, this should result in a lot of randomness in the chosen closest neighbor of the noisy observation representing $\vx$ in the high-dimensional space for all considered $\alpha\in\{2,\ldots,6\}$.
However, as discussed in Example \ref{hyperharmonic}, this will not be the case whenever $\alpha > 4$, as $\vx$ will very likely correctly choose $\vy$ as its neighbor even when this choice is affected by noise in high dimensions.
This is confirmed by the empirical probabilities (Figure \ref{AlphaExp}, Right), which furthermore agree with all limits obtained in Example \ref{hyperharmonic} from Corollary \ref{randomornot}.

\subsection{Learning with Random Neighbors}
\label{overfit}

\paragraph{Dimensionality Reduction} As also discussed in Section \ref{intro}, dimensionality reductions are commonly applied for preprocessing high-dimensional data that is corrupted by noise.
The obtained distances in the lower-dimensional space are then assumed to be more informative for inference and machine learning (Figure \ref{kNNrandom}).
This raises the question whether dimensionality reductions can naturally accommodate the effect of noise on high-dimensional neighboring relations.

To investigate this, consider a ground truth data set $\mX$ of $n$ evenly spaced points on the line segment $\mathcal{L}$ from the origin to $\vz^{(d)}$ in $\mathbb{R}^d$, where $\vz=\vz(\alpha)$ is as defined in Example \ref{hyperharmonic} by fixing some $\alpha\in \mathbb{R}_{>2}\cup\{\infty\}$ (Figure \ref{setupoverfit}, Left).
Since these points are evenly spaced on $\mathcal{L}$, the growth rate of all squared distances (and the differences between them) will be identical to the growth rate of $\|\vz^{(d)}\|^2$, up to some constant factor depending on the fixed ground truth ordering of the considered points.
Thus, from Corollary \ref{randomornot} we find that under the effect of noise in high dimensions, \emph{all} empirical neighbors will become random for $\alpha<4$, and \emph{all} empirical neighbors will likely remain truthful for $\alpha>4$.

\begin{figure}[b!]
	\centering
	\includegraphics[width=\linewidth]{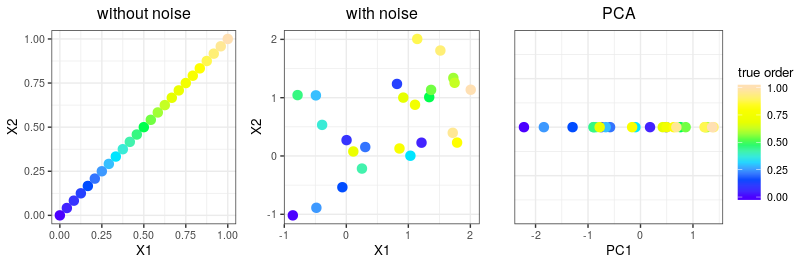}
	\caption{(Left) $n=25$ ground truth points, which make up $\mX$, are evenly spaced on a linear model $\mathcal{L}\subseteq\mathbb{R}^d$.
	(Middle) Rather than observing $\mX$, we observe $\mX+\mN$ for a random noise matrix $\mN$.
	(Right) A 1-dimensional PCA dimensionality reduction aims to retrieve the true ordering of the points on $\mathcal{L}$.}
	\label{setupoverfit}
	\includegraphics[width=\linewidth]{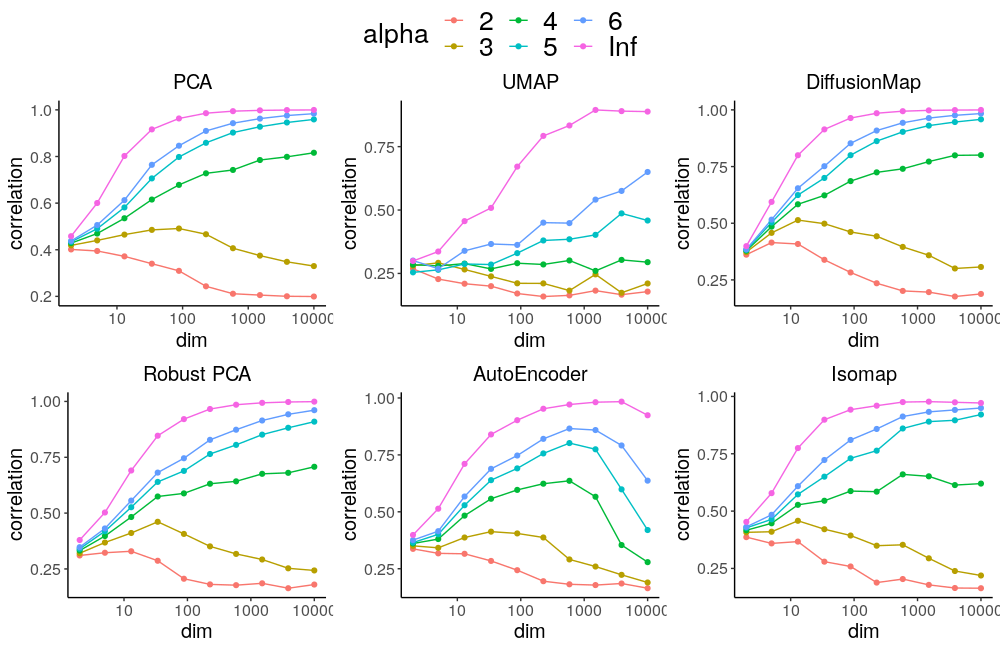}
	\caption{The performance of six common dimensionality reduction methods for recovering the ground truth ordering of the points on $\mathcal{L}$ under the effect of noise, by dimension and signal growth rate.}
	\label{expdimred}
\end{figure}

For a 1D-dimensionality reduction method $f$, we can now study how well $f$ is able to recover neighboring relations of $\mX$ from $\mX+\mN$, with $\mN$ a random noise matrix.
For this, we look at the correlation between the ordering of points on $\mX$ and on $f(\mX+\mN)$ (Figure \ref{setupoverfit}).
Furthermore, we investigate this for six different dimensionality reduction methods that are commonly used for noise or feature size reduction prior to visualization, (topological) inference, or clustering: PCA \cite{wold1987principal, van2008visualizing, slingshot, cannoodt2016computational}, UMAP \cite{mcinnes2018umap}, diffusion maps \cite{coifman2006diffusion, vandaele2020mining, cannoodt2016computational}, robust PCA \cite{candes2011robust} (a variant of PCA that assumes the data is composed in a low-rank component $\mX$ and a \emph{sparse} noise component $\mN$), a basic autoencoder \cite{kramer1991nonlinear, vincent2010stacked} with 5 hidden layers and $\tanh$ activation, and Isomap \cite{tenenbaum2000global}.
We evaluated their performances for $n=25$ points, up to $d=\num{10000}$ dimensions, growth rates determined by $\alpha\in\{2,3,4,5,6,\infty\}$, and averaged over 100 noise replicates from $\rn\sim \mathrm{U}[-1.25, 1.25]$ per dimension.
The autoencoder was built in Python. 
Other models ran under standard settings in \textsf{R} (with 10 neighbors instead of 50 for Isomap).
Figure \ref{expdimred} shows the results.

We consistently observe that the performances increase by dimension for $\alpha>4$, and decrease for $\alpha < 4$.
Following the previous results (Figure \ref{AlphaExp}, Right), this provides empirical evidence that the performance of these common dimensionality reduction methods is directly affected by whether noise causes randomness in high-dimensional neighborhood relations or not.
This thus suggests that these methods themselves may be susceptible to the noise they aim to reduce.
These observations are only contradicted by the autoencoder, which showed convergence issues for larger dimensions. 

Finally, the case $\alpha=4$ deserves special interest.
Since in a practical setting additional distributional conditions of $\mX$ will likely result in some non-extreme degree of randomness in the empirical neighborhood relations, we observe that this may be reflected in the performance of dimensionality reductions as well.

\paragraph{Spectral Clustering} Spectral clustering uses the spectrum of a similarity matrix from the data to perform a dimensionality reduction, prior to clustering the data in fewer dimensions \cite{filippone2008survey}.
Naturally, when the dimensionality reduction is affected by noise, so will the consecutive clustering performance.

\begin{figure}[b!]
	\centering
	\includegraphics[width=\linewidth]{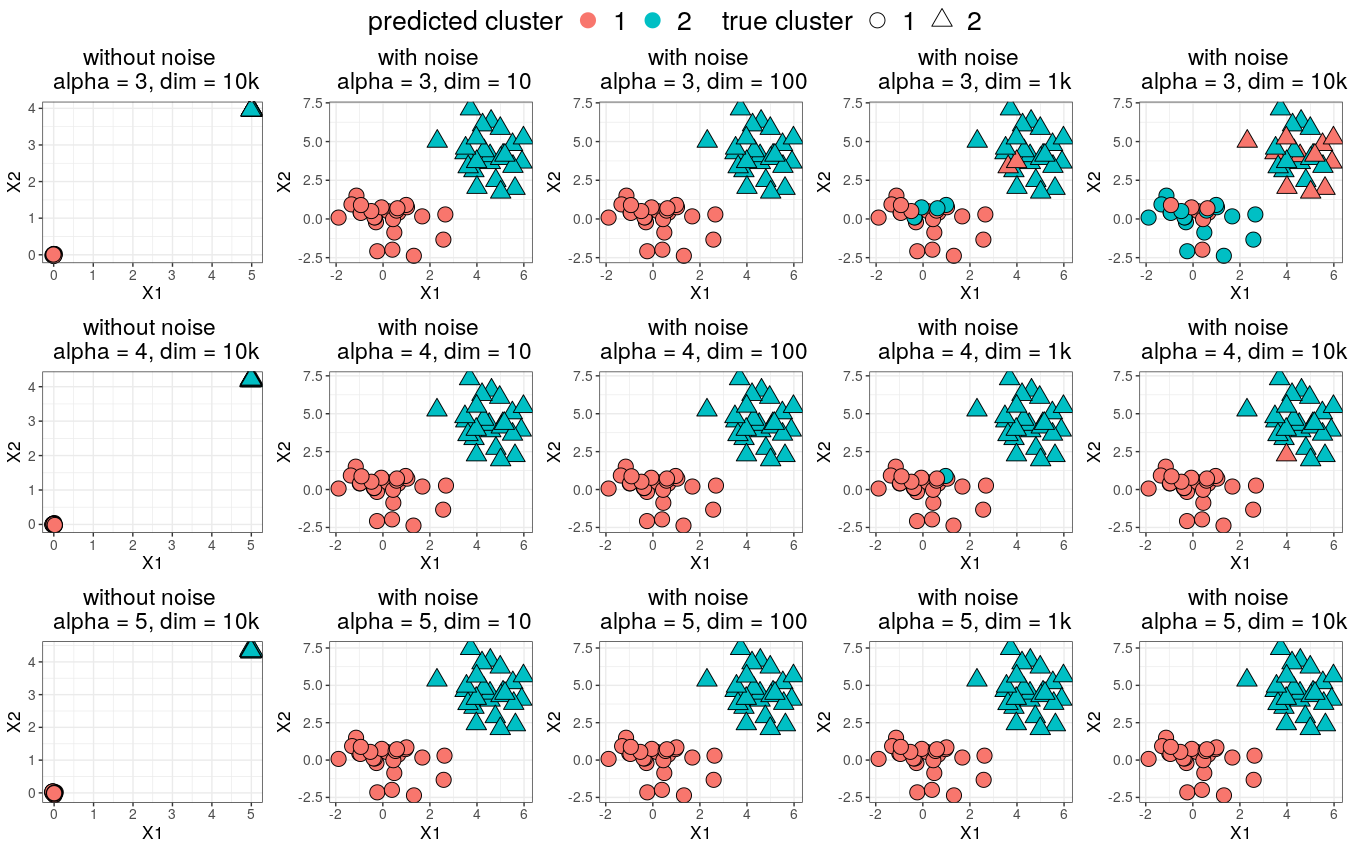}
	\caption{The results of spectral clustering for various input dimensionalities and values $\alpha$ that control the separation of two ground truth clusters in high dimensions, with and without Gaussian noise.
	Each plot is the restriction of the data to its first two coordinates.}
	\label{Clusters}
\end{figure}

To illustrate this, we constructed ground truth clusters by randomly sampling 25 data points from each of two $\num{10000}$-dimensional hyperballs $\mathcal{S}_0\coloneqq B\left(\vzero, \frac{\|\vc(\alpha)\|}{5}\right)$ and $\mathcal{S}_c\coloneqq B\left(\vc(\alpha), \frac{\|\vc(\alpha)\|}{5}\right)$, where $\vc(\alpha)$ is defined by
$$
    c_d(\alpha)=
    \begin{cases}
        \frac{5}{\sqrt[^\alpha]{d}}&\mbox{if }d\leq 2;\\
        \frac{1}{\sqrt[^\alpha]{d}}&\mbox{if }d>2.
    \end{cases}
$$
Hence, $\|\vc^{d}(\alpha)\|$ grows exactly as $\|\vz^{d}(\alpha)\|$ in Example \ref{hyperharmonic}.
The reason that we scale the first two coordinates is to provide more interpretable visualizations in Figure \ref{Clusters}.
Finally, we added noise sampled from the standard normal distribution $\mathcal{N}(0, 1)$ to each dimension.
We now ask the question how well spectral clustering is able to recover the ground truth clusters from the high-dimensional data, as shown in Figure \ref{Clusters}.

Due to the triangle inequality, any two points in the same ground truth cluster will be closer to each other than to any point in a different cluster, according to the ground truth distances.
As a consequence, without noise, spectral clustering (we use the standard settings from the \textsf{R} library Spectrum with a maximum of two clusters) recovers the clusters perfectly from the high-dimensional data (Figure \ref{Clusters}, left column).
Furthermore, also from the triangle inequality and our results in Section \ref{whensection}, it can be shown that when corrupted by noise, points will likely remain closer to points in the same ground truth cluster for $\alpha > 4$, whereas for $\alpha<4$, points will be nearly equally likely closer to points in other ground truth clusters than their own.

From Figure \ref{Clusters}, we see that this directly affects the performance of spectral clustering.
For $\alpha=5$, the clusters are perfectly identified, whereas for $\alpha=3$, the inferred clusters become increasingly meaningless when the dimensionality of the data from which they are derived grows. 
Again, $\alpha=4$ corresponds to a boundary case.
Here, the true clusters are identified well, although not perfectly, from the noisy high-dimensional data.

\begin{figure}[b!]
	\centering
	\includegraphics[width=\linewidth]{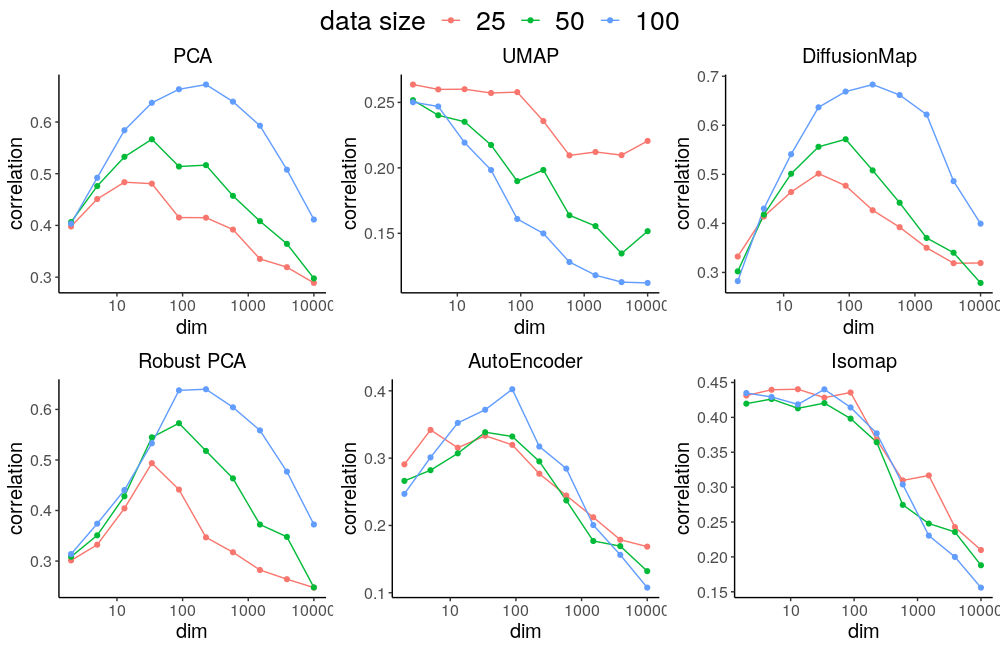}
	\caption{The performance of six common dimensionality reduction methods for recovering the ground truth ordering of the points on $\mathcal{L}$ under the effect of noise, by dimension and data size, for $\alpha=3$.}
	\label{DimredBySize}
\end{figure}

\paragraph{Role of Data Size} 

As common in machine learning applications, we observed that having more data can resolve much of the issues caused by noise.
Figure \ref{DimredBySize} illustrates this for our experiment summarized by Figures \ref{setupoverfit} \& \ref{expdimred}, where we now fixed $\alpha=3$, but varied the data size $n\in\{25,50,100\}$.
As discussed above, neighboring relations in the high-dimensional data will eventually become random for $\alpha=3$.
For all dimensionality reduction methods, we observe that their performance consistently drops for sufficiently high dimensions (Figure \ref{DimredBySize}).
However, for PCA, robust PCA, and diffusion maps, we observe that for larger data sizes, higher performances are reached first, and the dimension after which they struggle to recover the model gets delayed.
For the autoencoder and Isomap, the role of the data size is inconclusive. 
Interestingly, the performance with UMAP is consistently worse for larger data sizes (with the current settings).
These results can vary for the value of $\alpha$ however.
For example, in case of the autoencoder model which consistently showed an optimal dimension after which the performance decreases even for $\alpha>4$ (Figure \ref{expdimred}), we observed that larger data sizes may accommodate the noise (see Supplementary Figure \ref{AEalpha5} in Appendix \ref{suppfig}).

\section{Discussion and Conclusion}
\label{discandcon}

\emph{Noise can be, but does not have to be, fatal} when learning from high-dimensional data based on distances.
Although this is not a surprising fact, we provided a first and exact mathematical characterization when such distances become (un)informative under noise.
Furthermore, we found that our concept of meaningfulness of distances, i.e., when they are informative for the ground truth, is fundamentally different from distance concentration, and suggests direct connections to the ability of dimensionality reductions to recover the data model.
Although we focused on small artificial data sets to validate the results in this (mainly theoretical) paper, they are interesting nevertheless, and encourage further foundational and practical research into learning from noisy high-dimensional data.

Unfortunately, the conditions for distances to be meaningful will be difficult to assess in practice.
For example, one can easily derive from our results that when many features are irrelevant to the model, neighborhood relations will become uninformative in the presence of noise. 
In practice however, we may be unsure whether any features are irrelevant at all.
How algorithms may actually benefit from the results presented in this paper, is open to further research.
Nevertheless, there is an abundance of high-dimensional data where we cannot effectively recover the structure due to noise, leading to poor subsequent model inference, such as in biological single-cell data analysis.
We argue that better understanding the behavior of distances in noisy high-dimensional data---for which we provided, illustrated, and validated theoretical results in this paper---is imperative for one to be able to design better computational methods for their analysis.

\subsubsection*{Acknowledgments}

This research was funded by the ERC under the EU's 7th Framework and H2020 Programmes (ERC Grant Agreement no.\ 615517 and 963924), the Flemish Government (AI Research Program), and the FWO (project no.\ G091017N, G0F9816N, 3G042220).

\bibliography{references}

\clearpage
\appendix

\thispagestyle{empty}

% For one-column format, uncomment the following:
%\onecolumn \makesupplementtitle
% For two-column format, uncomment the following:
\twocolumn[ \makesupplementtitle ]

\section{Supplementary Figures}
\label{suppfig}

\begin{figure}[h]
  \centering
    \includegraphics[width=\linewidth]{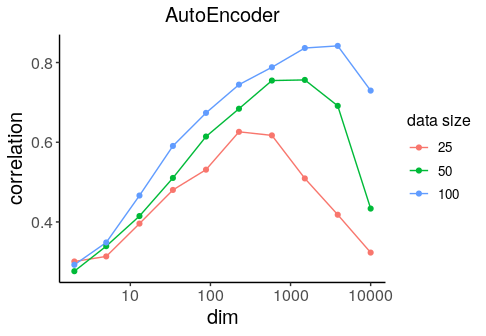}
    \caption{The performance of the autoencoder model for recovering the ground truth ordering of the points on the linear model $\mathcal{L}$ presented in the main paper (see also Figures \ref{setupoverfit} \& \ref{expdimred}) under the effect of noise, by dimension and data size, for $\alpha=5$.
    Neighboring relations will become increasingly truthful in the noisy high-dimensional data.
    Nevertheless---for the same hyperparameters and architecture apart from the outer layers (which accommodate the data dimension)---the autoencoder eventually struggles to recover the model from sufficiently large dimensions.
    However, larger data sizes may provide a temporary solution to this.}
    \label{AEalpha5}
\end{figure}

\section{Theorems and Proofs}
\label{thandproofs}

This part of the appendix contains the mathematical proofs of the results presented in the main paper.
The following is the proof of our main theorem that leads to all principal results presented in this paper. 

\begin{proof}[Proof of Theorem \ref{CLTspecial}]
	For $i\in\mathbb{N}^*$, denote $\delta_i(\vx,\vy)\coloneqq x_i-y_i$ and $\delta_i(\vx,\vz)\coloneqq x_i-z_i$.
	For each $d\in\mathbb{N}^*$, we have
	{\scriptsize
	    \begin{align*}
	    &\rz(d)\\
	    &\hspace{.25em}\coloneqq\left\|\rvn^{(d)}_{\vx}-\rvn^{(d)}_{\vy}+\vx^{(d)}-\vy^{(d)}\right\|^2-\left\|\rvn^{(d)}_{\vx}-\rvn^{(d)}_{\vz}+\vx^{(d)}-\vz^{(d)}\right\|^2\\
	    &\hspace{.25em}=\sum_{i=1}^d\underbrace{\left[\left(\underbrace{\rn_{x_i}-\rn_{y_i}+\delta_i(\vx,\vy)}_{\eqqcolon \rr_{x_i,y_i}}\right)^2-\left(\underbrace{\rn_{x_i}-\rn_{z_i}+\delta_i(\vx,\vz)}_{\eqqcolon \rr_{x_i,z_i}}\right)^2\right]}_{\eqqcolon \rz_i}.
	    \end{align*}}
	We have
	$$
	\mathbb{E}\left(\rr_{x_i,y_i}\right)=\delta_i(\vx,\vy),\hspace{1em}\mathrm{Var}\left(\rr_{x_i,y_i}\right)=2\sigma^2,
	$$
	so that
	$$
	\mathbb{E}\left(\rr_{x_i,y_i}^2\right)=\mathrm{Var}\left(\rr_{x_i,y_i}\right)+\mathbb{E}\left(\rr_{x_i,y_i}\right)^2=2\sigma^2+\delta^2_i(\vx,\vy),
	$$
	and thus
	$$
	\mathbb{E}(\rz_i)=\delta_i^2(\vx,\vy)-\delta_i^2(\vx,\vz).
	$$
	Since $\mathrm{Cov}(\rx,\ry)=\mathbb{E}(\rx\ry)-\mathbb{E}(\rx)\mathbb{E}(\ry)$ for random variables $\rx,\ry$, by symmetry and the fact that $\rn_{x_i}$ and $\rn_{y_i}$ are independent and $\mathbb{E}(\rn_{x_i}^{2k+1})=0$ for $k\in\mathbb{N}$, we have
	\begin{align*}
	\mathrm{Var}\left(\rr^2_{x_i,y_i}\right)&=2\mathrm{Var}\left(\rn_{x_i}^2\right)-8\mathrm{Cov}\left(\rn_{x_i}^2,\rn_{x_i}\rn_{y_i}\right)\\
	&\hspace{1em}+4\mathrm{Var}(\rn_{x_i}\rn_{y_i})+8\delta_i^2(x,y)\mathrm{Var}(\rn_{x_i})\\
	&=2\left(\mu'_4-\mathbb{E}\left(\rn_{x_i}^2\right)^2\right)+4\mathbb{E}\left(\rn_{x_i}^2\right)^2\\
	&\hspace{1em}+8\sigma^2\delta_i^2(\vx,\vy)\\
	&=2\mu'_4+2\sigma^4+8\sigma^2\delta_i^2(\vx,\vy)
	\end{align*}
	and analogously
	\begin{align*}
	&\mathrm{Cov}\left(\rr^2_{x_i,y_i},\rr^2_{x_i,z_i}\right)\\
	&\hspace{2em}=\mathrm{Var}\left(\rn_{x_i}^2\right)\\
	&\hspace{3em}-4\mathrm{Cov}\left(\rn_{x_i}^2,\rn_{x_i}\rn_{y_i}\right)\\
	&\hspace{3em}+2\left(\delta_i(\vx,\vy)+\delta_i(\vx,\vz)\right)\mathrm{Cov}\left(\rn_{x_i}^2,\rn_{x_i}\right)\\
	&\hspace{3em}+4\mathrm{Cov}(\rn_{x_i}\rn_{y_i},\rn_{x_i}\rn_{z_i})\\
	&\hspace{3em}-4\left(\delta_i(\vx,\vy)+\delta_i(\vx,\vz)\right)\mathrm{Cov}(\rn_{x_i}\rn_{y_i},\rn_{x_i})\\
	&\hspace{3em}+4\delta_i(\vx,\vy)\delta_i(\vx,\vz)\mathrm{Var}(\rn_{x_i})\\
	&\hspace{2em}=\mu'_4-\mathbb{E}\left(\rn_{x_i}^2\right)^2+4\sigma^2\delta_i(\vx,\vy)\delta_i(\vx,\vz)\\
	&\hspace{2em}=\mu'_4-\sigma^4+4\sigma^2\delta_i(\vx,\vy)\delta_i(\vx,\vz).
	\end{align*}
	It thus holds that
	\begin{align*}
	&\mathrm{Var}(\rz_i)\\
	&\hspace{2em}=\mathrm{Var}\left(\rr^2_{x_i,y_i}\right)+\mathrm{Var}\left(\rr^2_{x_i,z_i}\right)\\
	&\hspace{3em}-2\mathrm{Cov}\left(\rr^2_{x_i,y_i},\rr^2_{x_i,z_i}\right)\\
	&\hspace{2em}=2\mu'_4+6\sigma^4\\
	&\hspace{3em}+8\sigma^2\left(\delta^2_i(\vx,\vy)+\delta^2_i(\vx,\vz)-\delta_i(\vx,\vy)\delta_i(\vx,\vz)\right).
	\end{align*}
	We conclude that
	$$
	\mu\left(\rz(d)\right)=\sum_{i=1}^d\mathbb{E}(\rz_i)=\left\|\vx^{(d)}-\vy^{(d)}\right\|^2-\left\|\vx^{(d)}-\vz^{(d)}\right\|^2,
	$$
	and
	\begin{align*}
	\sigma^2\left(\rz(d)\right)&=\sum_{i=1}^d\mathrm{Var}(\rz_i)\\
	&=2d\left(\mu'_4+3\sigma^4\right)\\
	&\hspace{1em}+8\sigma^2\left(\left\|\vx^{(d)}-\vy^{(d)}\right\|^2+\left\|\vx^{(d)}-\vz^{(d)}\right\|^2\right.\\
	&\hspace{6em}\left.\vphantom{\left\|\vx^{(d)}\right|^2}-\left\langle \vx^{(d)} - \vy^{(d)}, \vx^{(d)} - \vz^{(d)}\right\rangle\right).
	\end{align*}
	Also observe that
	\begin{align*}
	&\left\|\vx^{(d)}-\vy^{(d)}\right\|^2+\left\|\vx^{(d)}-\vz^{(d)}\right\|^2\\
	&\hspace{1em}-\left\langle \vx^{(d)} - \vy^{(d)}, \vx^{(d)} - \vz^{(d)}\right\rangle\\
	&\hspace{2em}\geq \left\|\vx^{(d)}-\vy^{(d)}\right\|^2+\left\|\vx^{(d)}-\vz^{(d)}\right\|^2\\
	&\hspace{3em}-\left\|\vx^{(d)}-\vy^{(d)}\right\|\left\|\vx^{(d)}-\vz^{(d)}\right\|\\
	&\hspace{2em}= \left(\left\|\vx^{(d)}-\vy^{(d)}\right\|-\left\|\vx^{(d)}-\vz^{(d)}\right\|\right)^2\\
	&\hspace{3em}+\left\|\vx^{(d)}-\vy^{(d)}\right\|\left\|\vx^{(d)}-\vz^{(d)}\right\|\\
	&\hspace{2em}\geq \left\|\vx^{(d)}-\vy^{(d)}\right\|\left\|\vx^{(d)}-\vz^{(d)}\right\|,
	\end{align*}
	which is used for Remark \ref{otherconditions} in the main paper.
	Now for $\epsilon > 0$, let
	$$
	\sA_{i,\epsilon}\coloneqq\left\{(\rn_{x_i},\rn_{y_i},\rn_{z_i}):\left|\rz_{i}-\mathbb{E}(\rz_i)\right|>\epsilon \sigma\left(\rz(d)\right)\right\},
	$$
	and
	$$
	C_d(\epsilon)\coloneq \frac{1}{12}\left(\sqrt{\underbrace{36\Delta_{\infty}^2(d)+24\epsilon\sigma\left(\rz(d)\right)}_{\eqqcolon D}}-6\Delta_{\infty}(d)\right).
	$$
	For any $\epsilon>0$ and $M\geq \sqrt{\frac{6}{\epsilon}}$, by (\ref{condition}), we find that for $d$ sufficiently large
	\begin{align}
	\label{usedbelow}
	C_d(\epsilon)&\geq\frac{\sqrt{24\epsilon}-\frac{6}{M}}{12}\sqrt{\sigma\left(\rz(d)\right)}\\
	&\geq\frac{2\sqrt{6\epsilon}-\sqrt{6\epsilon}}{12}\sqrt{\sigma\left(\rz(d)\right)}=\frac{\sqrt{6\epsilon}}{12}\sqrt{\sigma\left(\rz(d)\right)},
	\end{align}
	so that in particular $\lim_{d\rightarrow\infty}C_d(\epsilon)=+\infty$.
	If now 
	\begin{align*}
	\max\{|\rn_{x_i}|,|\rn_{y_i}|,|\rn_{z_i}|\}\leq C_d(\epsilon),
	\end{align*}
	we find that for $i=1,\ldots,d$,
	\begin{align*}
	&|\rz_i-\mathbb{E}(\rz_i)|\\
	&\hspace{2em}\leq |\rn_{y_i}|^2+|\rn_{z_i}|^2+2|\rn_{x_i}||\rn_{y_i}|+2|\rn_{x_i}||\rn_{z_i}|\\
	&\hspace{3em}+2(|\rn_{x_i}|+|\rn_{y_i}|+|\rn_{z_i}|)\Delta_{\infty}(d)\\
	&\hspace{2em}\leq 6C_d(\epsilon)^2+6\Delta_{\infty}(d)C_d(\epsilon)\\
	&\hspace{2em}\leq\frac{1}{24}\left(72\Delta_{\infty}^2(d)+24\epsilon\sigma\left(\rz(d)\right)-12\Delta_{\infty}(d)\sqrt{D}\right)\\
	&\hspace{3em}+\frac{1}{2}\left(\Delta_{\infty}(d)\sqrt{D}-6\Delta_{\infty}^2(d)\right)\\
	&\hspace{2em}=\epsilon\sigma\left(\rz(d)\right).
	\end{align*}
	This shows that
	{\small	
	\begin{align*}
	&\sA_{i,\epsilon}\subseteq \{(\rn_{x_i},\rn_{y_i},\rn_{z_i}):\max\{|\rn_{x_i}|,|\rn_{y_i}|,|\rn_{z_i}|\}>C_d(\epsilon)\}\\
	&\hspace{20em}\eqqcolon \widetilde{\sA}_{i,\epsilon}.
	\end{align*}}
	Observe that for every $i,k,l\in\mathbb{N}^*$, we have
	\begin{align*}
	&\mathbb{E}\left(\rn^{2k+1}_{x_i}\rn^{l}_{y_i}\mathds{1}_{\widetilde{\sA}_{i,\epsilon}}\right)\\
	&\hspace{2em}=\mathbb{E}\left(\rn^{2k+1}_{x_i}\right)\mathbb{E}\left(\rn^{l}_{y_i}\right)\\
	&\hspace{3em}-\mathbb{E}\left(\rn^{2k+1}_{x_i}\mathds{1}_{\left(\widetilde{\sA}_{i,\epsilon}\right)^c}\right)\mathbb{E}\left(\rn^{l}_{y_i}\mathds{1}_{\left(\widetilde{\sA}_{i,\epsilon}\right)^c}\right)\\
	&\hspace{2em}=0.
	\end{align*}
	Again, due to symmetry, we thus have
	\begin{align*}
	&\mathrm{E}\left(\left(\rr_{x_i,y_i}^2-\rr_{x_i,z_i}^2\right)\mathds{1}_{\widetilde{\sA}_{i,\epsilon}}\right)\\
	&\hspace{2em}=\left(\delta^2_i(\vx,\vy)-\delta^2_i(\vx,\vz)\right)P\left(\widetilde{\sA}_{i,\epsilon}\right).
	\end{align*}
	Furthermore, we have
	\begin{align*}
	\mathbb{E}\left(\rr_{x_i,y_i}^4\mathds{1}_{\widetilde{\sA}_{i,\epsilon}}\right)&=2\mathbb{E}\left(\rn_{x_i}^4\mathds{1}_{\widetilde{\sA}_{i,\epsilon}}\right)+6\mathbb{E}\left(\rn_{x_i}^2\rn_{y_i}^2\mathds{1}_{\widetilde{\sA}_{i,\epsilon}}\right)\\
	&\hspace{2em}+12\delta^2_i(\vx,\vy)\mathbb{E}\left(\rn_{x_i}^2\mathds{1}_{\widetilde{\sA}_{i,\epsilon}}\right)\\
	&\hspace{2em}+\delta_i(\vx,\vy)^4P\left(\widetilde{\sA}_{i,\epsilon}\right),
	\end{align*}
	and
	\begin{align*}
	&\mathbb{E}\left(\left(\rr_{x_i,y_i}^2\rr_{x_i,z_i}^2\right)\mathds{1}_{\widetilde{\sA}_{i,\epsilon}}\right)\\
	\hspace{2em}&=\mathbb{E}\left(\rn_{x_i}^4\mathds{1}_{\widetilde{\sA}_{i,\epsilon}}\right)+3\mathbb{E}\left(\rn_{x_i}^2\rn_{y_i}^2\mathds{1}_{\widetilde{\sA}_{i,\epsilon}}\right)\\
	&\hspace{3em}+2\left(\delta_i(\vx,\vy)+\delta_i(\vx,\vz)\right)^2\mathbb{E}\left(\rn_{x_i}^2\mathds{1}_{\widetilde{\sA}_{i,\epsilon}}\right)\\
	&\hspace{3em}+\delta_i(\vx,\vy)^2\delta_i(\vx,\vz)^2P\left(\widetilde{\sA}_{i,\epsilon}\right).
	\end{align*}
	Putting things together, we have
	\begin{align*}
	&\mathbb{E}\left((\rz_{i}-\mathbb{E}(\rz_i))^2\mathds{1}_{\sA_{i,\epsilon}}\right)\\
	&\hspace{.15em}\leq \mathbb{E}\left(\rz_{i}-\mathbb{E}(\rz_i))^2\mathds{1}_{\widetilde{\sA}_{i,\epsilon}}\right)\\
	&\hspace{.15em}=\mathbb{E}\left(\left(\rr_{x_i,y_i}^2-\rr_{x_i,z_i}^2\right)^2\mathds{1}_{\widetilde{\sA}_{i,\epsilon}}\right)\\
	&\hspace{1.15em}-2\mathbb{E}(\rx_i)\mathbb{E}\left(\left(\rr_{x_i,y_i}^2-\rr_{x_i,z_i}^2\right)\mathds{1}_{\widetilde{\sA}_{i,\epsilon}}\right)\\
	&\hspace{1.15em}+\mathbb{E}(\rz_i)^2P\left(\widetilde{\sA}_{i,\epsilon}\right)\\
	&\hspace{.15em}=\mathbb{E}\left(\rr_{x_i,y_i}^4\mathds{1}_{\widetilde{\sA}_{i,\epsilon}}\right)+\mathbb{E}\left(\rr_{x_i,z_i}^4\mathds{1}_{\widetilde{\sA}_{i,\epsilon}}\right)\\
	&\hspace{1.15em}-2\mathbb{E}\left(\left(\rr_{x_i,y_i}^2\rr_{x_i,z_i}^2\right)\mathds{1}_{\widetilde{\sA}_{i,\epsilon}}\right)\\
	&\hspace{1.15em}+\mathbb{E}(\rz_i)P\left(\widetilde{\sA}_{i,\epsilon}\right)\left(\mathbb{E}(\rz_i)-2\left(\delta^2_i(\vx,\vy)-\delta^2_i(\vx,\vz)\right)\right)\\
	&\hspace{.15em}=2\mathbb{E}\left(\rn_{x_i}^4\mathds{1}_{\widetilde{\sA}_{i,\epsilon}}\right)+6\mathbb{E}\left(\rn_{x_i}^2\rn_{y_i}^2\mathds{1}_{\widetilde{\sA}_{i,\epsilon}}\right)\\
	&\hspace{1.15em}+8\left(\delta_i^2(\vx,\vy)+\delta^2_i(\vx,\vz)-\delta_i(\vx,\vy)\delta_i(\vx,\vz)\right)\\
	&\hspace{14em}\times\mathbb{E}\left(\rn_{x_i}^2\mathds{1}_{\widetilde{\sA}_{i,\epsilon}}\right)\\
	&\hspace{1.15em}+\delta^2_i(\vx,\vy)\delta^2_i(\vx,\vz)P\left(\widetilde{\sA}_{i,\epsilon}\right).
	\end{align*}
	Now since for each $j,k,l\in\mathbb{N}$, it holds that
	\begin{align*}
	&\mathbb{E}\left(\rn_{x_i}^{2j} \rn_{y_i}^{2k} \rn_{z_i}^{2l}\mathds{1}_{\widetilde{\sA}_{i,\epsilon}}\right)\\
	&\hspace{2em}\leq\mathbb{E}\left(\rn_{x_i}^{2j}\mathds{1}_{|\rn_{x_i}|>C_d(\epsilon)}\right)\mathbb{E}\left(\rn_{y_i}^{2k}\right)\mathbb{E}\left(\rn_{z_i}^{2l}\right)\\
	&\hspace{3em}+\mathbb{E}\left(\rn_{y_i}^{2k}\mathds{1}_{|\rn_{y_i}|>C_d(\epsilon)}\right)\mathbb{E}\left(\rn_{x_i}^{2j}\right)\mathbb{E}\left(\rn_{z_i}^{2l}\right)\\
	&\hspace{3em}+\mathbb{E}\left(\rn_{z_i}^{2l}\mathds{1}_{|\rn_{z_i}|>C_d(\epsilon)}\right)\mathbb{E}\left(\rn_{x_i}^{2j}\right)\mathbb{E}\left(\rn_{y_i}^{2k}\right),
	\end{align*}
	we find that
	\begin{align*}
	&\mathbb{E}\left((\rz_{i}-\mathbb{E}(\rz_i))^2\mathds{1}_{\sA_{i,\epsilon}}\right)\\
	&\hspace{.15em}\leq 2\mathbb{E}\left(\rn_{x_i}^4\mathds{1}_{|\rn_{x_i}|>C_d(\epsilon)}\right)\\	
	&\hspace{1.15em}+4\mu'_4P\left(|\rn_{x_i}|>C_d(\epsilon)\right)\\
	&\hspace{1.15em}+12\sigma^2\mathbb{E}\left(\rn_{x_i}^2\mathds{1}_{|\rn_{x_i}|>C_d(\epsilon)}\right)\\
	&\hspace{1.15em}+6\sigma^4P\left(|\rn_{x_i}|>C_d(\epsilon)\right)\\
	&\hspace{1.15em}+8\left(\delta_i^2(\vx,\vy)+\delta^2_i(\vx,\vz)-\delta_i(\vx,\vy)\delta_i(\vx,\vz)\right)\\
	&\hspace{3em}\times\left(\mathbb{E}\left(\rn_{x_i}^2\mathds{1}_{|\rn_{x_i}|>C_d(\epsilon)}\right)\!+\!2\sigma^2P\left(|\rn_{x_i}|>C_d(\epsilon)\right)\right)\\
	&\hspace{1.15em}+3\delta^2_i(\vx,\vy)\delta^2_i(\vx,\vz)P\left(|\rn_{x_i}|>C_d(\epsilon)\right)\\
	&\hspace{.15em}\leq 2\mathbb{E}\left(\rn_{x_i}^4\mathds{1}_{|\rn_{x_i}|>C_d(\epsilon)}\right)+12\sigma^2\mathbb{E}\left(\rn_{x_i}^2\mathds{1}_{|\rn_{x_i}|>C_d(\epsilon)}\right)\\
	&\hspace{1.15em}+8\left(\delta_i^2(\vx,\vy)+\delta^2_i(\vx,\vz)-\delta_i(\vx,\vy)\delta_i(\vx,\vz)\right)\\
	&\hspace{3em}\times\mathbb{E}\left(\rn_{x_i}^2\mathds{1}_{|\rn_{x_i}|>C_d(\epsilon)}\right)\\
	&\hspace{1.15em}+\left(4\mu'_4+6\sigma^4+16\max\left\{1,\sigma^2\right\}M_i\right)\\
	&\hspace{3em}\times P\left(|\rn_{x_i}|>C_d(\epsilon)\right),
	\end{align*}
	where
	\begin{align*}
	M_i&\coloneqq \delta_i^2(\vx,\vy)+\delta^2_i(\vx,\vz)-\delta_i(\vx,\vy)\delta_i(\vx,\vz)\\
	&\hspace{1em}+\delta_i^2(\vx,\vy)\delta^2_i(\vx,\vz).	    
	\end{align*}
	Summing over $i=1,\ldots,d$, we find that
	\begin{align*}
	&\frac{1}{\sigma^2\left(\rz(d)\right)}\sum_{i=1}^d\mathbb{E}\left(\rn_{x_i}^4\mathds{1}_{|\rn_{x_i}|>C_d(\epsilon)}\right)\\
	&\hspace{2em}\leq\frac{d\mathbb{E}\left(\rn_{x_1}^4\mathds{1}_{\left|\rn_{x_1}\right|>C_d(\epsilon)}\right)}{\sigma^2\left(2d(\mu'_4+3\sigma^4)\right)}\overset{d\rightarrow\infty}{\longrightarrow} 0,
	\end{align*}
	since $\mu'_4$ is finite and $C_d(\epsilon)\overset{d\rightarrow\infty}{\longrightarrow} +\infty$.
	Analogously, we have
	\begin{align*}
	\frac{1}{\sigma^2\left(\rz(d)\right)}\sum_{i=1}^d\mathbb{E}\left(\rn_{x_i}^2\mathds{1}_{|\rn_{x_i}|>C_d(\epsilon)}\right)\overset{d\rightarrow\infty}{\longrightarrow} 0.
	\end{align*}
	We also have
	\begin{align*}
	&\frac{1}{\sigma^2\left(\rz(d)\right)}\sum_{i=1}^d\left(\delta_i^2(\vx,\vy)+\delta^2_i(\vx,\vz)-\delta_i(\vx,\vy)\delta_i(\vx,\vz)\right)\\
	&\hspace{12.5em}\times\mathbb{E}\left(\rn_{x_i}^2\mathds{1}_{|\rn_{x_i}|>C_d(\epsilon)}\right)\\
	&\hspace{1em}\leq\frac{1}{8\sigma^2}\mathbb{E}\left(\rn_{x_1}^2\mathds{1}_{\left|\rn_{x_1}\right|>C_d(\epsilon)}\right)\overset{d\rightarrow\infty}{\longrightarrow} 0.
	\end{align*}
	Finally, we have that
	\begin{align*}
	\frac{1}{\sigma^2\left(\rz(d)\right)}\sum_{i=1}^d M_i P\left(|\rn_{x_i}|>C_d(\epsilon)\right)\overset{d\rightarrow\infty}{\longrightarrow} 0.
	\end{align*}
	Indeed, given the observation above, it suffices to show that
	$$
	P\left(|\rn_{x_1}|>C_d(\epsilon)\right)\frac{\sum_{i=1}^d\delta^2_i(\vx,\vy)\delta^2_i(\vx,\vz)}{\sigma^2\left(\rz(d)\right)}\overset{d\rightarrow\infty}{\longrightarrow} 0.
	$$
	Using Markov's inequality, (\ref{condition}), and (\ref{usedbelow}), this follows from the fact that for $d$ sufficiently large
	\begin{align*}
	&P\left(|\rn_{x_1}|>C_d(\epsilon)\right)\frac{\sum_{i=1}^d\delta^2_i(\vx,\vy)\delta^2_i(\vx,\vz)}{\sigma^2\left(\rz^{(d)}\right)}\\
	&\hspace{2em}\leq \frac{\mathbb{E}\left(|\rn_{x_1}|\right)d\Delta^4_{\infty}(d)}{C_d(\epsilon)\sigma^2\left(\rz(d)\right)}\\
	&\hspace{2em}\leq \frac{12\mathbb{E}\left(|\rn_{x_1}|\right)d\Delta^4_{\infty}(d)}{\sqrt{6\epsilon}\sigma^{\frac{5}{2}}\left(\rz(d)\right)}\overset{d\rightarrow\infty}{\longrightarrow} 0.
	\end{align*}
	We conclude that
	$$
	\lim_{d\rightarrow\infty}\frac{1}{\sigma^2\left(\rz(d)\right)}\mathbb{E}\left((\rz_{i}-\mathbb{E}(\rz_i))^2\mathds{1}_{\sA_{i,\epsilon}}\right)=0,
	$$
	and this for every $\epsilon>0$.
	Hence, Linderberg's condition is satisfied, so that we may apply the central limit theorem to $\rz(d)$, i.e.,
	$$
	\frac{\rz(d)-\mu\left(\rz(d)\right)}{\sigma\left(\rz(d)\right)}\overset{d\rightarrow\infty}{\longrightarrow}_p \mathcal{N}(0,1).
	$$
	Since convergence in probability to a continuous distribution function necessarily implies uniform convergence \cite{parzen1960modern}, we find that
	\begin{align*}
	&\lim_{d\rightarrow \infty}\left|P\left(\rz(d)\leq 0\right)-\Phi\left(-\frac{\mu\left(\rz(d)\right)}{\sigma\left(\rz(d)\right)}\right)\right|\\
	&\hspace{1em}=\lim_{d\rightarrow \infty}\left|P\left(\frac{\rz(d)-\mu\left(\rz(d)\right)}{\sigma\left(\rz(d)\right)}\leq -\frac{\mu\left(\rz(d)\right)}{\sigma\left(\rz(d)\right)}\right)\right.\\
	&\hspace{12.5em}\left.-\Phi\left(-\frac{\mu\left(\rz(d)\right)}{\sigma\left(\rz(d)\right)}\right)\right|= 0,
	\end{align*}
	which concludes the proof.
\end{proof}

Theorem \ref{CLTspecial} can now be used to easily derive Corollary \ref{randomornot} as follows.

\begin{proof}[Proof of Corollary \ref{randomornot}]
	\begin{enumerate}
		\item 
		This is an immediate consequence of Theorem \ref{CLTspecial}.
		
		\item 
		By assumption, it holds that
		$$
		\frac{d}{\left(\left\|\vx^{(d)}-\vz^{(d)}\right\|^2-\left\|\vx^{(d)}-\vy^{(d)}\right\|^2\right)^2}\overset{d\rightarrow\infty}{\longrightarrow}0.
		$$
		Furthermore, we have
		\begin{align*}
		&0\leq\tfrac{\left\|\vx^{(d)}-\vy^{(d)}\right\|^2+\left\|\vx^{(d)}-\vz^{(d)}\right\|^2-\left\langle \vx^{(d)} - \vy^{(d)}, \vx^{(d)} - \vz^{(d)}\right\rangle}{\left(\left\|\vx^{(d)}-\vz^{(d)}\right\|^2-\left\|\vx^{(d)}-\vy^{(d)}\right\|^2\right)^2}\\
		&\leq\frac{\left(\left\|\vx^{(d)}-\vy^{(d)}\right\|+\left\|\vx^{(d)}-\vz^{(d)}\right\|\right)^2}{\splitfrac{\left(\left\|\vx^{(d)}-\vz^{(d)}\right\|-\left\|\vx^{(d)}-\vy^{(d)}\right\|\right)^2}{\times\left(\left\|\vx^{(d)}-\vz^{(d)}\right\|+\left\|\vx^{(d)}-\vy^{(d)}\right\|\right)^2}}\\
		&=\frac{1}{\left(\left\|\vx^{(d)}-\vz^{(d)}\right\|-\left\|\vx^{(d)}-\vy^{(d)}\right\|\right)^2}.
		\end{align*}
		Now choose any $M>0$.
		We know that for $d$ sufficiently large
		\begin{align*}
		&\left\|\vx^{(d)}-\vz^{(d)}\right\|-\left\|\vx^{(d)}-\vy^{(d)}\right\|\\
		\hspace{2em}&=\frac{\left\|\vx^{(d)}-\vz^{(d)}\right\|^2-\left\|\vx^{(d)}-\vy^{(d)}\right\|^2}{\left\|\vx^{(d)}-\vz^{(d)}\right\|+\left\|\vx^{(d)}-\vy^{(d)}\right\|}\\
		&\hspace{2em}\geq\frac{2MC\sqrt{d}}{2\Delta_{\infty}(d)\sqrt{d}}\geq M,
		\end{align*}
		so that $\frac{1}{\left(\left\|\vx^{(d)}-\vz^{(d)}\right\|-\left\|\vx^{(d)}-\vy^{(d)}\right\|\right)^2}\overset{d\rightarrow\infty}{\longrightarrow}0$.
		Hence, $\zeta^{(d)}\left(\mu'_4,\sigma,\vx^{(d)},\vy^{(d)},\vz^{(d)}\right)\overset{d\rightarrow\infty}{\longrightarrow}\infty$. 
		The result now follows from Theorem \ref{CLTspecial}.
	\end{enumerate}
\end{proof}

\end{document}